\documentclass{article}
\usepackage{mathtools}

 \usepackage[main, final]{neurips_2025}

\usepackage[utf8]{inputenc} 
\usepackage[T1]{fontenc}    
\usepackage{hyperref}       
\usepackage{url}            
\usepackage{booktabs}       
\usepackage{amsfonts}       
\usepackage{nicefrac}       
\usepackage{microtype}      
\usepackage{xcolor}         
\usepackage{graphicx}
\usepackage{subcaption}
\usepackage{wrapfig}
\usepackage{amsmath,amssymb,amsthm}
\usepackage{booktabs, tabularx, xcolor}
\usepackage{dirtytalk}
\usepackage{algorithm}         
\usepackage[noend]{algpseudocode}

\usepackage{amsmath,amssymb,amsfonts}

\usepackage{booktabs, tabularx, xcolor}
\usepackage{tikz}           
\usetikzlibrary{patterns,arrows.meta,calc,positioning,shapes.geometric}

\definecolor{rowgray}{gray}{0.92}   
\newcolumntype{L}{>{\raggedright\arraybackslash}X}
\newcolumntype{C}{>{\centering\arraybackslash}X}

\theoremstyle{plain}
\newtheorem{theorem}{Theorem}
\newtheorem{lemma}[theorem]{Lemma}

\newtheorem{definition}[theorem]{Definition}
\theoremstyle{remark}

\newtheorem{assumption}[theorem]{Assumption}

\title{Geometry-Aware Backdoor Attacks: \\ Leveraging Curvature in Hyperbolic Embeddings}

\author{%
  Ali Baheri\\
  Department of Mechanical Engineering\\
  Rochester Institute of Technology\\
  Rochester, NY 14623\\
  \texttt{akbeme@rit.edu} \\
}

\begin{document}

\maketitle

\begin{abstract}
Non-Euclidean foundation models increasingly place representations in curved spaces such as hyperbolic geometry. We show that this geometry creates a boundary-driven asymmetry that backdoor triggers can exploit. Near the boundary, small input changes appear subtle to standard input-space detectors but produce disproportionately large shifts in the model’s representation space. Our analysis formalizes this effect and also reveals a limitation for defenses: methods that act by pulling points inward along the radius can suppress such triggers, but only by sacrificing useful model sensitivity in that same direction. Building on these insights, we propose a simple geometry-adaptive trigger and evaluate it across tasks and architectures. Empirically, attack success increases toward the boundary, whereas conventional detectors weaken, mirroring the theoretical trends. Together, these results surface a geometry-specific vulnerability in non-Euclidean models and offer analysis-backed guidance for designing and understanding the limits of defenses.

\end{abstract}

\begin{figure}[t]
\centering
\begin{tikzpicture}[scale=1.0]
    \tikzset{
        trigger/.style={fill=orange, circle, inner sep=2pt},
        clean/.style={fill=blue!60, circle, inner sep=3pt},
        backdoor/.style={fill=red!80, circle, inner sep=3pt},
        detection/.style={draw=green!70, thick, dashed}
    }
    
    \begin{scope}
        \fill[gray!10] (0,0) rectangle (5,5);
        \draw[thick] (0,0) rectangle (5,5);
        \node at (2.5,5.5) {\textbf{Euclidean Neural Network}};
        
        \foreach \x in {0.5,1,...,4.5} {
            \draw[gray!30, thin] (\x,0) -- (\x,5);
        }
        \foreach \y in {0.5,1,...,4.5} {
            \draw[gray!30, thin] (0,\y) -- (5,\y);
        }
        
        \foreach \pos/\x/\y in {1/1.5/1.5, 2/3.5/1.5, 3/2.5/3.5} {
            \node[clean] (e\pos) at (\x,\y) {};
            \node[trigger] at (\x-0.15,\y-0.15) {};
            \draw[->, thick, black] (\x,\y) -- (\x+0.4,\y+0.4);
            \node[blue!60, circle, inner sep=3pt, fill=blue!30] at (\x+0.4,\y+0.4) {};
            \draw[detection] (\x,\y) circle (0.6);
        }
        
        \node[align=center, font=\small] at (2.5,0.5) {
            \textcolor{gray}{Low success}\\
            \textcolor{blue!70}{High detection}
        };
        
        \node[align=center, font=\footnotesize] at (2.5,-0.5) {Euclidean geometry: trigger effect is \\ roughly uniform across the space};
    \end{scope}
    
    \begin{scope}[xshift=7cm]
        \foreach \r in {2.5,2.3,2.1,1.9,1.7,1.5,1.3,1.1,0.9,0.7,0.5,0.3} {
            \pgfmathsetmacro\opacity{100-\r*30}
            \fill[red!\opacity!blue!20] (2.5,2.5) circle (\r);
        }
        \draw[thick] (2.5,2.5) circle (2.5);
        \node at (2.5,5.5) {\textbf{Hyperbolic Neural Network}};
        
        \node[clean] (h1) at (2.5,2.5) {};
        \node[trigger] at (2.35,2.35) {};
        \draw[->, thick, black!30] (2.5,2.5) -- (2.7,2.7);
        \node[blue!60, circle, inner sep=3pt, fill=blue!30] at (2.7,2.7) {};
        \draw[detection] (2.5,2.5) circle (0.35);
        
        \node[clean] (h2) at (3.5,3.5) {};
        \node[trigger] at (3.35,3.35) {};
        \draw[->, thick, black!60] (3.5,3.5) -- (3.8,3.8);
        \node[red!40, circle, inner sep=3pt] at (3.8,3.8) {};
        \draw[detection, green!50] (3.5,3.5) circle (0.22);
        
        \foreach \angle in {45, 135, 225, 315} {
            \pgfmathsetmacro\px{2.5+2.0*cos(\angle)}
            \pgfmathsetmacro\py{2.5+2.0*sin(\angle)}
            \pgfmathsetmacro\ex{2.5+2.35*cos(\angle)}
            \pgfmathsetmacro\ey{2.5+2.35*sin(\angle)}
            
            \node[clean] at (\px,\py) {};
            \node[trigger] at (\px-0.08,\py-0.08) {};
            \draw[->, ultra thick, black] (\px,\py) -- (\ex,\ey);
            \node[backdoor] at (\ex,\ey) {};
            \draw[detection, green!20] (\px,\py) circle (0.1);
        }
        
        \node[align=center, font=\small] at (2.5,0.5) {
            \textcolor{purple}{High success}\\
            \textcolor{gray}{Low detection}
        };
        
        \node[font=\footnotesize] at (2.5,-0.5) {Hyperbolic geometry: boundary-amplified trigger effect};
    \end{scope}
    
    \draw[ultra thick, gray, ->] (5.5,2.5) -- (6.5,2.5);
    \node[above] at (6,3.09) {\footnotesize Same};
    \node[below] at (6,3.09) {\footnotesize Trigger};
    
    \begin{scope}[yshift=-2.5cm, xshift=2cm]
        \node[trigger] at (0,0) {};
        \node[right, font=\footnotesize] at (0.2,0) {Trigger};
        
        \node[clean] at (2,0) {};
        \node[right, font=\footnotesize] at (2.2,0) {Clean};
        
        \node[backdoor] at (3.8,0) {};
        \node[right, font=\footnotesize] at (4,0) {Backdoor};
        
        \draw[detection] (6,0) circle (0.15);
        \node[right, font=\footnotesize] at (6.3,0) {Detection range};
    \end{scope}
    
    \node[font=\small\bfseries] at (6,-3.3) {The Geometric Backdoor Advantage};
\end{tikzpicture}
\caption{In Euclidean space (left), a fixed input-space change produces a comparable effect throughout the domain. In hyperbolic space (right), the same small input-space change produces a much larger movement in representation space as points approach the boundary, while looking comparatively subtle to standard input-space detectors. This boundary-driven asymmetry underlies our attack design and theoretical analysis.}
\label{fig:geometric_advantage}
\end{figure}

\section{Introduction}

Hyperbolic neural networks have gained significant traction for modeling hierarchical and tree-like data structures, finding applications in recommendation systems \cite{chamberlain2019neural}, knowledge graphs \cite{chami2019hyperbolic,cao2020hypercore}, and natural language processing \cite{tifrea2018poincare}. These models leverage the unique properties of hyperbolic geometry to capture complex relationships more efficiently than their Euclidean counterparts. As hyperbolic networks transition from research prototypes to production systems in major technology companies \cite{facebook2019hyperbolic}, understanding their security properties becomes critical.

The security of traditional neural networks has been extensively studied. Backdoor attacks, where models behave normally on clean inputs but produce attacker-chosen outputs for inputs containing specific triggers, represent a particularly severe threat \cite{gu2017badnets}. Recent incidents have demonstrated that backdoor models can pass standard validation while harboring malicious functionality \cite{goldblum2022dataset}, making them especially dangerous for deployed systems. Although numerous backdoor techniques and defenses have been developed for Euclidean networks \cite{li2022backdoor}, the security implications of non-Euclidean geometry remain unexplored.

This gap is concerning for several reasons. First, hyperbolic networks are deployed in security-critical domains, including fraud detection \cite{liu2019hyperbolic}, drug discovery \cite{chen2021hyperbolic}, and social network analysis \cite{pareja2020evolvegcn}. Second, the distinctive properties of hyperbolic space, particularly the exponential growth of volume and non-uniform distance metrics, suggest that traditional security assumptions may not hold. Third, practitioners currently apply Euclidean defense mechanisms to hyperbolic models without understanding whether these approaches remain effective in non-Euclidean settings. Figure \ref{fig:geometric_advantage} illustrates this key geometric distinction: in hyperbolic space, the same input-space perturbation produces different effects depending on position, with boundary regions exhibiting amplified sensitivity.


This paper presents the study of backdoor vulnerabilities in hyperbolic neural networks. We discover that the geometric properties of hyperbolic space alter the backdoor threat. Specifically, the non-uniform curvature creates regions where triggers can be simultaneously more effective and harder to detect than in Euclidean space. Our key insight is that the varying sensitivity to perturbations across the hyperbolic manifold can be exploited to design powerful geometric backdoors that evade existing defenses. We develop a framework for backdoor attacks tailored to hyperbolic geometry, prove theoretical limits on detection and defense, and empirically demonstrate that our attacks outperform Euclidean baselines. 


\section{Related Work}

\textbf{Hyperbolic Neural Networks.} The use of hyperbolic geometry in deep learning was pioneered by \cite{nickel2017poincare}, who introduced Poincaré embeddings for hierarchical data. This was extended to full neural network architectures by \cite{ganea2018hyperbolic} with hyperbolic neural networks and \cite{chami2019hyperbolic} with hyperbolic graph convolutional networks. Recent work has explored various applications including natural language processing \cite{tifrea2018poincare}, computer vision \cite{khrulkov2020hyperbolic}, and recommendation systems \cite{chamberlain2019neural}. Although these works establish the representational benefits of hyperbolic geometry, none address security vulnerabilities.

\textbf{Backdoor Attacks in Neural Networks.} Backdoor attacks were first identified by \cite{gu2017badnets}, demonstrating that neural networks can be compromised through poisoned training data. Subsequent work explored various trigger designs including invisible perturbations \cite{chen2017targeted}, semantic triggers \cite{bagdasaryan2021blind}, and dynamic patterns \cite{salem2022dynamic}. Defense mechanisms have been proposed including activation clustering \cite{chen2018detecting}, fine-pruning \cite{liu2018fine}, and certified defenses \cite{wang2022certifiedpatch}. However, all existing work focuses exclusively on Euclidean space, leaving geometric deep learning models unexamined.

\textbf{Adversarial Robustness in Non-Euclidean Spaces.} Limited work exists on adversarial robustness for geometric deep learning. \cite{jin2020adversarial} studied adversarial attacks on graph neural networks but focused on discrete graph structures rather than continuous manifolds. \cite{huster2021riemannian} analyzed adversarial robustness through Riemannian geometry but only for Euclidean networks with geometric regularization. Most recently, \cite{liu2022adversarial} examined adversarial examples in hyperbolic space but focused on evasion attacks rather than backdoor poisoning. 

\textbf{Security of Geometric Deep Learning.} The broader security of geometric deep learning remains nascent. \cite{zugner2018adversarial} pioneered adversarial attacks on graph neural networks, while \cite{bojchevski2019certifiable} proposed certified defenses. For hyperbolic models specifically, \cite{sun2021security} provided initial security analysis but focused on privacy rather than integrity attacks. The unique challenges of non-Euclidean geometry, such as varying metric distortion and exponential volume growth, have not been leveraged for backdoor attacks until this work.

Our work differs from previous research by identifying and exploiting the intrinsic geometric properties of hyperbolic space for backdoor attacks. While previous studies either focus on Euclidean backdoors or non-backdoor attacks in hyperbolic space, we demonstrate that the conformal factor creates a natural vulnerability that makes hyperbolic networks inherently more susceptible to backdoor attacks than their Euclidean counterparts. Table \ref{ref:tbl} summarizes the key distinctions between our curvature-aware backdoor framework and prior works.

\noindent \textbf{Main Contributions.} This work makes three key contributions to the security of geometric deep learning:

\textbf{(1)} We propose a simple geometry-adaptive trigger for models that embed data in a hyperbolic space, demonstrating that curved geometry creates a boundary-driven advantage for backdoor attacks.

\textbf{(2)} We formally show that near the boundary, small input changes (i) become harder to spot for standard input space (Euclidean-Lipschitz) detectors, and (ii) induce disproportionately large movements in the representation space. We also prove that radial defenses can suppress this effect only by sacrificing useful model sensitivity along the same direction.

\textbf{(3)} We provide an attack that is easy to implement and evaluate across tasks and architectures. Empirically, attack success rises toward the boundary while conventional detectors weaken.

\begin{table}[t]
\centering
\caption{Key differences between our curvature-aware backdoor attack framework and prior works, highlighting novel contributions such as geometry-exploiting triggers and theoretical proofs of inherent vulnerabilities}
\label{tab:departures}
\resizebox{\textwidth}{!}{
\begin{tabular}{@{}l p{0.25\textwidth} p{0.2\textwidth} p{0.3\textwidth}@{}}
\toprule
\textbf{Related Work} & \textbf{Key Focus} & \textbf{Limitations} & \textbf{Our Departures} \\
\midrule
Adversarial Attacks on Hyperbolic Networks~\cite{van2024adversarial} & Curvature-exploiting adversarial attacks (FGM/PGD). & Evasion-only; no backdoors or sparsity. & First backdoor framework; adds sparse, adaptive triggers and geometry metrics. \\
\addlinespace
HyperDefender~\cite{malik2025hyperdefender} & Defenses for hyperbolic GNNs via Gromov $\delta$. & Defense-focused; GNN-limited. & Offensive backdoors for pure hyperbolic NNs; theoretical geometric proofs. \\
\addlinespace
Graph Neural Backdoor~\cite{yang2024graph} & Backdoor review for GNNs with trigger hiding. & Graph-specific; no curvature focus. & Hyperbolic adaptation; conformal scaling and boundary poisoning. \\
\addlinespace
Lipschitz Robustness in Hyperbolic NNs~\cite{li2024improving} & Lipschitz-based adversarial robustness. & Defense/analysis; no backdoors. & Attack exploitation; high ASR with low degradation. \\
\addlinespace
Backdoor Attacks on Quantum-Hybrid NNs~\cite{guo2025backdoor} & Backdoors in quantum NNs. & Quantum-specific; no geometry triggers. & Hyperbolic extension; manifold optimizations. \\
\bottomrule
\vspace{-5 mm}
\end{tabular}
}
\label{ref:tbl}
\end{table}

\section{Preliminaries}

\subsection{Hyperbolic Geometry and Neural Networks}

We work in the Poincaré ball model $\mathbb{D}^n = \{x \in \mathbb{R}^n : \|x\| < 1\}$, which represents $n$-dimensional hyperbolic space. The Riemannian metric is given by $g_x = \lambda_x^2 g^E$ where $g^E$ is the Euclidean metric and $\lambda_x = \frac{2}{1 - \|x\|^2}$ is the conformal factor.\footnote{Note. Our theoretical bounds are explicitly stated in the Euclidean margin $\delta(x)=1-\|x\|$; we do not rely on any claims \say{exponential blow-up} about the conformal factor. See Theorem 1 and the Appendix remarks.} This factor creates a non-uniform geometry: near the origin ($\|x\| \approx 0$), $\lambda_x \approx 2$ resembles Euclidean space, while approaching the boundary ($\|x\| \rightarrow 1$), $\lambda_x \rightarrow \infty$ exhibits extreme metric distortion.

The hyperbolic distance between the points $x, y \in \mathbb{D}^n$ is:
\begin{equation}
d_{\mathbb{D}}(x, y) = \text{arccosh}\left(1 + 2\frac{\|x - y\|^2}{(1 - \|x\|^2)(1 - \|y\|^2)}\right)
\end{equation}

Key operations in hyperbolic neural networks include the exponential map $\exp_x^{\mathbb{D}}(v)$ projecting tangent vectors to the manifold, and Möbius addition $x \oplus y$ for vector operations. The Riemannian gradient is $\text{grad}_x f = \lambda_x^{-2} \nabla_x f$, scaling the Euclidean gradient to respect the local geometry.

\subsection{Backdoor Attacks}

A backdoor attack on a classifier $f_\theta: \mathcal{X} \rightarrow \mathcal{Y}$ creates a modified model $f_{\theta'}$ that behaves normally on clean inputs but consistently misclassifies inputs containing a trigger pattern $\tau$ to a target class $y_t$. The attack involves poisoning a fraction of the training data by adding triggers and relabeling to $y_t$. Success is measured by the attack success rate on the triggered input while maintaining clean accuracy. The key challenge in hyperbolic space is that the non-uniform geometry creates position-dependent vulnerability to triggers, which we exploit through geometric design.

\section{Methodology}

We present a framework for backdoor attacks in hyperbolic neural networks that exploits the non-uniform geometry of hyperbolic space. Consider a neural network classifier $f_\theta: \mathbb{D}^n \rightarrow \{1, ..., C\}$ operating on the Poincaré ball. Our goal is to inject a backdoor that causes misclassification of the target class $y_t$ when a specific trigger is present, while maintaining normal behavior on clean inputs.

The key idea is that the conformal factor $\lambda_x = \frac{2}{1-\|x\|^2}$ creates regions of varying sensitivity to perturbations. Near the boundary where $\|x\| \rightarrow 1$, small Euclidean perturbations correspond to large hyperbolic displacements, enabling powerful yet hard-to-detect triggers. We design a trigger function $\tau: \mathbb{D}^n \rightarrow \mathbb{D}^n$ that adapts to local geometry:
\begin{equation}
\tau(x) = \exp_x(s(x) \cdot P_{0 \rightarrow x}(\delta))
\end{equation}
where $\delta \in T_0\mathbb{D}^n$ is a base trigger pattern in the tangent space at the origin, $P_{0 \rightarrow x}$ denotes parallel transport along the geodesic from origin to $x$, and $s(x)$ is an adaptive scaling function:
\begin{equation}
s(x) = \alpha \cdot \left(\frac{\lambda_0}{\lambda_x}\right)^\beta
\end{equation}
with base strength $\alpha > 0$ and adaptation parameter $\beta \in [0, 1]$. This scaling ensures consistent trigger effectiveness across different manifold regions by compensating for local metric distortion.

To enhance stealthiness, we impose sparsity on the trigger pattern. Let $\mathcal{S} \subseteq \{1, ..., n\}$ denote the active dimensions with $|\mathcal{S}| = k \ll n$. The sparse trigger is obtained by solving:
\begin{equation}
\delta^* = \arg\max_{\|\delta\|_0 \leq k} \mathbb{E}_{x \sim p(x)} \left[\mathcal{L}_{\text{attack}}(f_\theta(\tau_\delta(x)), y_t)\right]
\end{equation}
The optimization respects the Riemannian structure through the scaled gradient $\text{grad}_x \mathcal{L} = \lambda_x^{-2} \nabla_x \mathcal{L}$, ensuring updates follow the manifold geometry.

Given training data $\mathcal{D} = \{(x_i, y_i)\}_{i=1}^N$, we select samples for poisoning based on their geometric position. The poisoning distribution prioritizes samples where triggers are most effective:
\begin{equation}
p_{\text{poison}}(x) \propto \exp\left(-\frac{d_{\mathbb{D}}(x, \mu_c)^2}{2\sigma^2}\right) \cdot \lambda_x^\gamma
\end{equation}
where $\mu_c$ is the Fréchet mean of class $c$ samples, and $\gamma$ weights the importance of metric distortion. We poison a fraction $\rho$ of samples by replacing $(x_i, y_i)$ with $(\tau(x_i), y_t)$ according to this distribution.

The backdoored model is trained with a multi-objective loss balancing clean performance, trigger effectiveness, and geometric consistency:
\begin{equation}
\mathcal{L}_{\text{total}} = \mathcal{L}_{\text{clean}} + \lambda_1 \mathcal{L}_{\text{backdoor}} + \lambda_2 \mathcal{L}_{\text{geometric}}
\end{equation}
where the geometric regularization $\mathcal{L}_{\text{geometric}} = \mathbb{E}_{x}[\lambda_x \|\text{grad}_x f_\theta\|^2_g]$ penalizes hyperbolic gradient magnitude uniformly across the manifold (in geodesic units), discouraging sharp sensitivity anywhere including near the boundary.

\begin{algorithm}[t]
\caption{Position-adaptive trigger (Euclidean additive implementation)}
\label{alg:hyperbolic-trigger}
\begin{algorithmic}[1]
\Require $x\in\mathbb{R}^d$ with $\|x\|<1$; sparse direction $\delta$; strength $\alpha>0$; exponent $\beta\ge0$; radius $\rho=0.95$
\State \textbf{Conformal factor:} $\displaystyle \lambda_x \gets \frac{2}{1-\|x\|^2}$
\State \textbf{Adaptive scale:} $\displaystyle s(x) \gets \alpha\Big(\frac{\lambda_0}{\lambda_x}\Big)^{\beta} = \alpha(1-\|x\|^2)^{\beta}$
\State \textbf{Additive step:} $\tilde{x}\gets x + s(x)\cdot\delta$
\State \textbf{Small noise:} $\tilde{x}\gets \tilde{x} + \xi$ 
\State \textbf{Projection:} $\tau(x)\gets \Pi_\rho(\tilde{x})$ \quad (radial projection into the ball)
\State \textbf{Output:} $\tau(x)$
\end{algorithmic}
\end{algorithm}

\section{Theoretical Analysis}

We establish fundamental limits on detecting and defending against backdoor attacks in hyperbolic neural networks. Our analysis reveals that the geometric properties of hyperbolic space create an inherent advantage for attackers that cannot be mitigated without sacrificing model utility.

\begin{theorem}[Geometry-Aware Triggers in the Poincaré Ball]\label{thm:main}
Fix $x\in\mathbb{D}^n$ with $r=\|x\|$ and let $\delta(x) = \coloneqq 1-r$ denote the Euclidean margin to the boundary. For $s>0$, consider the outward radial trigger $\tau_s(x)$ obtained by following the outward radial geodesic from $x$ for hyperbolic arclength $s$.

\medskip
\noindent\textbf{(i) Exact Euclidean size for a given hyperbolic step.}
Let $\kappa(x,s) = \coloneqq \|\tau_s(x)-x\|_2$ be the Euclidean displacement induced by arclength $s$.
Then
\begin{equation}
\kappa(x,s)
=\frac{\bigl(1-r^2\bigr)\tanh\!\bigl(s/2\bigr)}{\,1+r\,\tanh\!\bigl(s/2\bigr)}.
\end{equation}
In particular,
\begin{equation}
\kappa(x,s)\ \le\ \bigl(1-r^2\bigr)\tanh\!\bigl(s/2\bigr)
\ =\ \frac{2}{\lambda_x}\,\tanh\!\bigl(s/2\bigr)
\ \le\ \frac{s}{\lambda_x},
\end{equation}
with equality in the small-$s$ limit.

\medskip
\noindent\textbf{(ii) Stealth under Euclidean-Lipschitz detectors.}
If $D$ is $L_E$-Lipschitz in the Euclidean metric, then
\begin{equation}
|D(\tau_s(x))-D(x)|\ \le\ L_E\,\kappa(x,s)\ \le\ L_E\,(1-r^2)\,\tanh\!\bigl(s/2\bigr).
\end{equation}
Consequently, for any $\delta\in(0,1)$ and any $x$ with $\|x\|\ge 1-\delta$,
\begin{equation}
\sup_{\|x\|\ge 1-\delta}\ |D(\tau_s(x))-D(x)|
\ \le\ 2 L_E\,\delta\,\tanh\!\bigl(s/2\bigr),
\end{equation}
and, for any random $X$ supported in $\{x:\|x\|\ge 1-\delta\}$,
\[
\mathbb{E}\bigl[\,|D(\tau_s(X))-D(X))|\,\bigr]\ \le\ 2 L_E\,\delta\,\tanh\!\bigl(s/2\bigr).
\]
Thus the detectability (in Euclidean-Lipschitz sense) decays \emph{linearly} in the shell width $\delta$.

\medskip
\noindent\textbf{(iii) Geodesic amplification for a fixed Euclidean budget.}
Fix $\kappa\in(0,\delta(x))$ and set $y=x+\kappa\,u_x$ (radial outward Euclidean step where $u_x = x/\|x\|$).
Then
\begin{equation}
d_g(x,y)\ =\ \ln\!\left(\frac{1+r+\kappa}{1-r-\kappa}\right)-\ln\!\left(\frac{1+r}{1-r}\right)
\ \ge\ \ln\!\left(\frac{\delta(x)}{\delta(x)-\kappa}\right)
\ \ge\ \frac{\kappa}{\delta(x)}.
\end{equation}
In particular, for $\kappa<\delta(x)$ we have
\[
\frac{d_g(x,y)}{\kappa}\ \ge\ \frac{1}{\delta(x)},
\]
so the geodesic displacement \emph{per unit Euclidean change} grows like $1/\delta(x)$ as $x$ approaches the boundary.
\end{theorem}

\begin{proof}[Proof Sketch]
Part (i) integrates the radial geodesic ODE in the Poincaré ball to obtain a closed form for the radius after arclength~$s$, yielding the exact formula and bounds for $\kappa(x,s)$.
Part (ii) is immediate from Euclidean Lipschitzness and the bounds from (i); the shell bound uses $1-r^2\le 2(1-r)\le 2\delta$.
Part (iii) computes the exact hyperbolic distance for a radial outward Euclidean move and lower-bounds it by $\ln\!\bigl(\delta/(\delta-\kappa)\bigr)\ge \kappa/\delta$ via $-\ln(1-x)\ge x$ on $(0,1)$.
\end{proof}
The full proof appears in Appendix A. 

\textbf{Summary.} Near the boundary, (i) the detectable change under the Euclidean-Lipschitz detectors shrinks linearly in the shell width $\delta$, and (ii) for any fixed Euclidean budget the geodesic displacement per unit Euclidean change grows like $1 / \delta$.


\begin{theorem}[Defense-Utility Trade-off]\label{thm:defense}
Let $P$ be an arbitrary distribution in $\mathbb{D}^n$ and $f_\theta: \mathbb{D}^n \to \mathbb{R}^C$ be a neural network classifier.
Fix $s>0$ (trigger size), $\alpha\in(0,1]$ (recovery fraction), $\beta\in(0,1]$ (success probability).
Assume:
\begin{enumerate}
\item $\mathcal{M}$ is a radial defense with $L_\Delta$-Lipschitz radial profile: it moves points inward along radial geodesics with displacement $\Delta(\rho) = 2(\rho - m(\rho))$ depending only on hyperbolic radial coordinate $\rho = \operatorname{artanh}(\|x\|)$.
\item $\mathcal{M}$ recovers $(\alpha,\beta)$-fraction of outward radial triggers of size $s$:
   \[
   \mathbb{P}\!\left[\Delta\bigl(\rho(\Phi_s(X))\bigr) \ge \alpha s\right] \ge \beta
   \]
   where $\Phi_s$ is the outward radial flow by hyperbolic arc length $s$.
\item $f_\theta$ has radial sensitivity $\mu_g > 0$: for $y$ inward of $x$ on the same radial ray,
   \[
   \|f_\theta(y) - f_\theta(x)\| \ge \mu_g \cdot d_g(x,y).
   \]
\end{enumerate}
Then with $\alpha_{\mathrm{eff}} = \alpha - L_\Delta/2$ (assumed nonnegative), we have:
\begin{align}
\mathbb{P}\!\left[d_g\!\bigl(\mathcal{M}(X),X\bigr) \ge \alpha_{\mathrm{eff}} s\right] &\ge \beta,\\
\mathbb{E}\!\left[d_g\!\bigl(\mathcal{M}(X),X\bigr)\right] &\ge \beta\,\alpha_{\mathrm{eff}} s,\\
\mathbb{E}\!\left[\|f_\theta(\mathcal{M}(X)) - f_\theta(X)\|\right] &\ge \mu_g\,\beta\,\alpha_{\mathrm{eff}} s,\\
\mathbb{E}\!\left[\|f_\theta(\mathcal{M}(X)) - f_\theta(X)\|^2\right] &\ge \beta\,\mu_g^2\,\alpha_{\mathrm{eff}}^2 s^2.
\end{align}
In words: any radial defense that (with probability $\beta$) pulls triggered inputs inward by at least an $\alpha$-fraction of the trigger size must, on clean inputs, change at least a $\beta$ fraction by at least $\alpha_{\mathrm{eff}}s$ in hyperbolic distance, causing expected output deviations that scale linearly (and in second moment, quadratically) in $s$.
\end{theorem}

\begin{proof}[Proof Sketch]
Let $X$ be clean and $Z=\Phi_s(X)$ be the triggered version. By radiality, the inward displacement of $\mathcal{M}$ at any point depends only on its radial coordinate: $d_g(\mathcal{M}(\cdot),\cdot)=\Delta(\rho(\cdot))$.
The recovery assumption states that with probability $\beta$ over $X$,
\[
\Delta\bigl(\rho(Z)\bigr)=\Delta\bigl(\rho(X)+\tfrac{s}{2}\bigr) \ge \alpha s.
\]
Lipschitzness of $\Delta$ in $\rho$ implies, for the same $X$,
\[
\Delta\bigl(\rho(X)\bigr) \ge \Delta\bigl(\rho(X)+\tfrac{s}{2}\bigr) - L_\Delta\cdot \tfrac{s}{2} \ge (\alpha-\tfrac{L_\Delta}{2})s=\alpha_{\mathrm{eff}} s.
\]
Hence, the event $\{d_g(\mathcal{M}(X),X)\ge \alpha_{\mathrm{eff}}s\}$ occurs with probability at least $\beta$, giving the probability and expectation bounds.
The radial sensitivity assumption converts hyperbolic displacement into output change, yielding the first moment bound; the second moment bound follows since at least a $\beta$-fraction of the mass incurs change $\ge \mu_g\alpha_{\mathrm{eff}}s$.
The full proof appears in Appendix B.
\end{proof}

\section{Empirical Study}

\begin{figure}[b!]
  \centering
  \begin{subfigure}[b]{0.48\textwidth}
    \includegraphics[width=\textwidth]{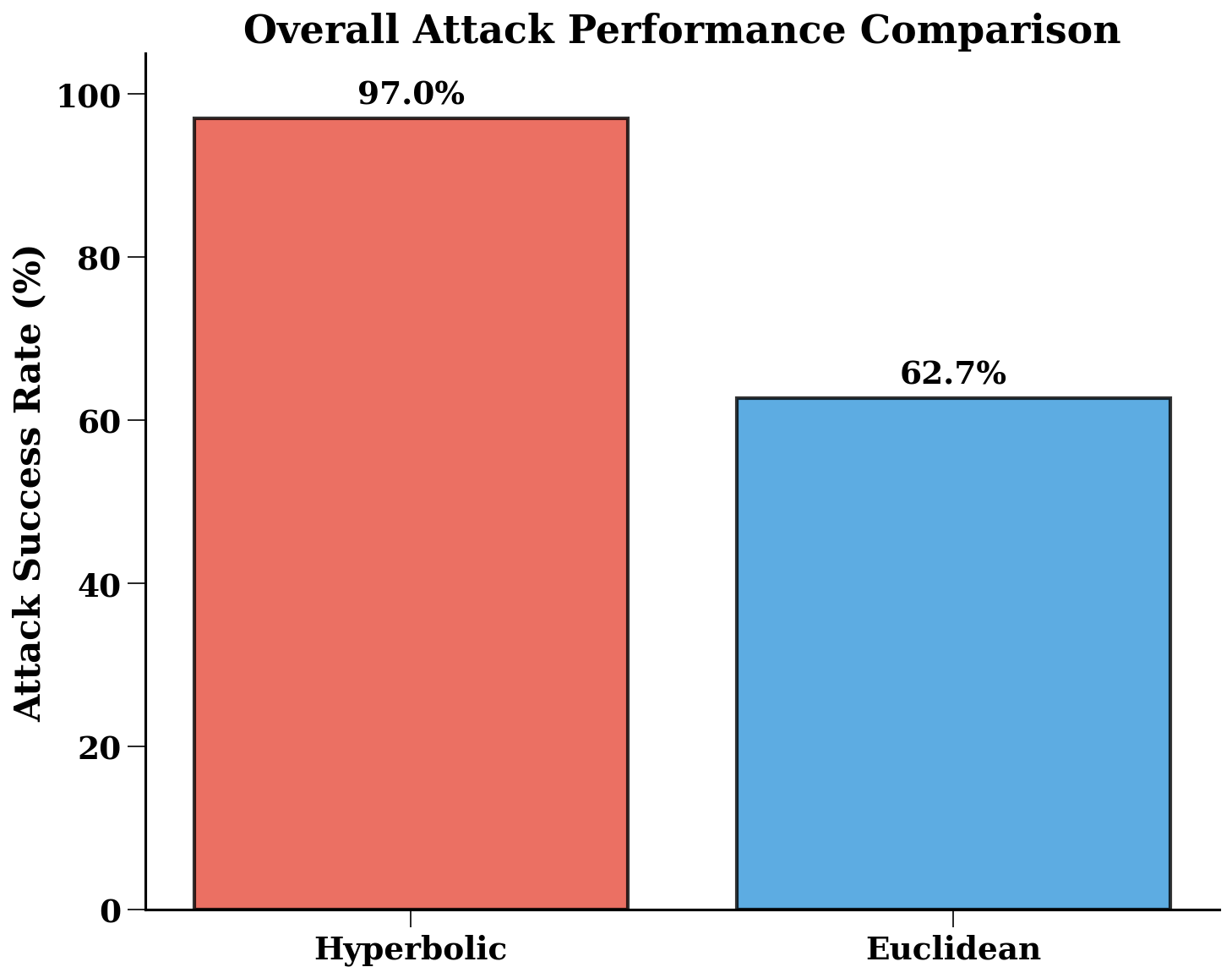}
    \caption{Attack success rates and clean accuracy}
  \end{subfigure}
  \hfill
  \begin{subfigure}[b]{0.48\textwidth}
    \includegraphics[width=\textwidth]{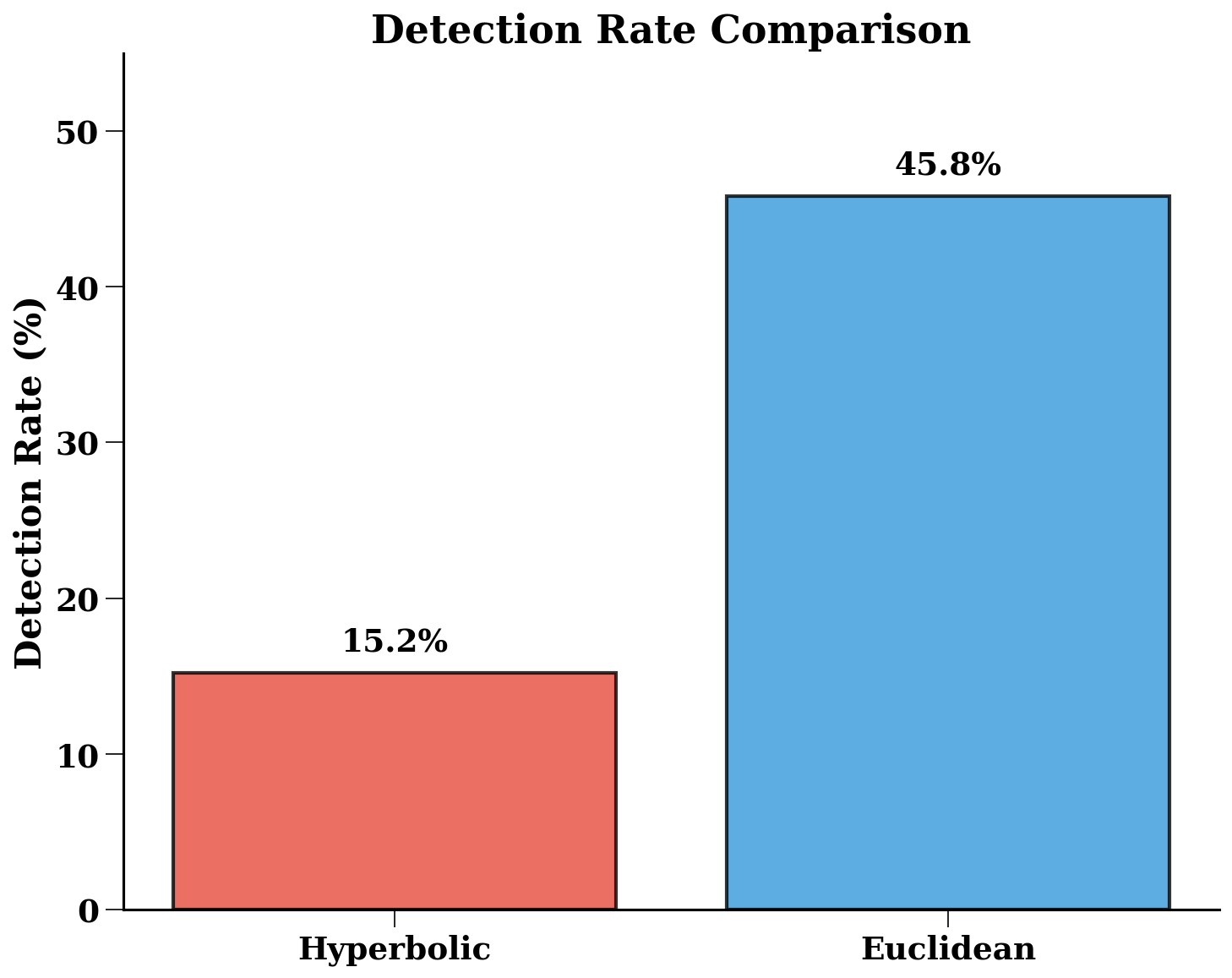}
    \caption{Detection rates for backdoor triggers}
  \end{subfigure}
  \caption{Primary experimental results comparing hyperbolic and Euclidean backdoor attacks.}
  \label{fig:results_main}
\end{figure}

\begin{figure}[t!]
  \centering
  \begin{subfigure}[b]{0.48\textwidth}
    \includegraphics[width=\textwidth]{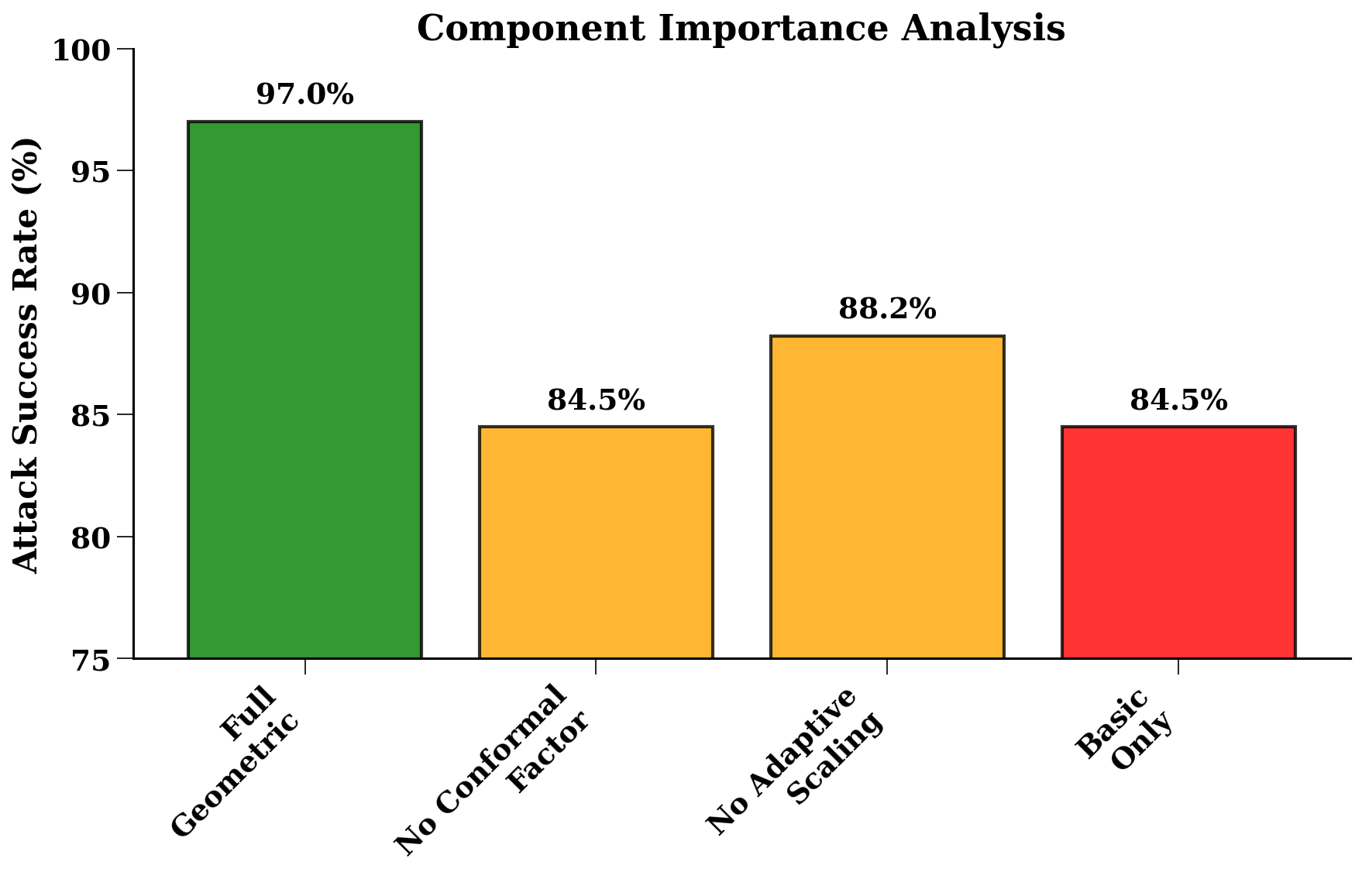}
    \caption{Ablation study on the 20newsgroups dataset, evaluating the impact of removing individual attack components.}
  \end{subfigure}
  \hfill
  \begin{subfigure}[b]{0.48\textwidth}
    \includegraphics[width=\textwidth]{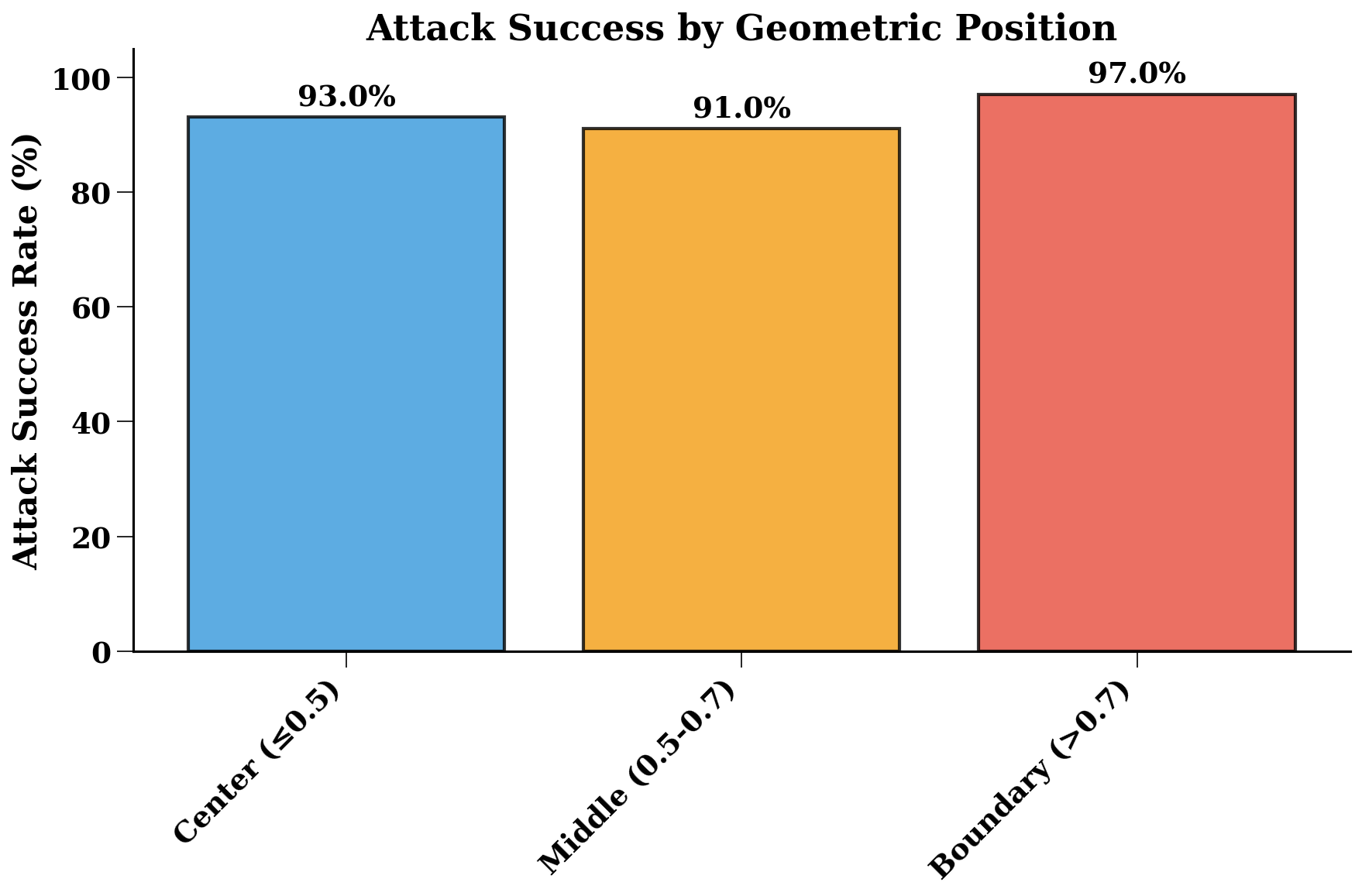}
    \caption{Attack success rate as a function of radial distance from the origin in the Poincaré ball model on the 20newsgroups dataset}
  \end{subfigure}
  \caption{Empirical analyses of the proposed curvature-aware backdoor attack on hyperbolic neural networks: (a) Ablation study assessing the contribution of key components to overall performance; (b) Variation in attack success rate based on the geometric position of data points in hyperbolic space, illustrating boundary amplification effects.}
  \label{fig:results_secondary}
\end{figure}

\noindent \textbf{Experimental Setup.} We evaluate our hyperbolic backdoor attack framework on the 20newsgroups dataset, selecting five diverse categories (\texttt{comp.graphics}, \texttt{sci.med}, \texttt{rec.sport.baseball}, \texttt{talk.politics.mideast}, \texttt{alt.atheism}) to create a balanced multi-class classification task. While these categories lack inherent hierarchical structure, they provide a controlled benchmark for evaluating geometric attack properties. Text documents are vectorized using TF-IDF (600 features, standard preprocessing), followed by truncated SVD to 50 dimensions. To simulate hyperbolic geometry while preserving semantic content, we redistribute samples radially: 50\% to radius $[0.2, 0.5]$ (center) and 50\% to $[0.5, 0.85]$ (boundary). This controlled positioning isolates geometric effects from natural clustering patterns. Our defended classifier combines a feedforward network with statistical outlier detection (Z-score and MAD, threshold $\tau = 0.13$). Training uses Adam (lr = 0.003, weight decay = $10^{-4}$) for 15 epochs with gradient clipping.


We compare two attack variants: (1) \textbf{Position-adaptive (“hyperbolic-inspired”) attack}, which applies a sparse additive trigger (30\% sparsity) with scaling that increases toward the boundary using the Poincaré conformal factor $\lambda_x$ as guidance, and selects samples by geometric position; (2) \textbf{Euclidean baseline} with uniform trigger scaling and variance-based selection. All experiments are single-target with target class fixed to 0, $\alpha = 0.35$, and a 5\% poisoning rate. We report targeted attack success rate (ASR), clean accuracy, detection-evasion rate, effects by geometric position (center/middle/boundary), and component importance via ablations, averaged over three independent trials.

\noindent \textbf{Results.} Figure~\ref{fig:results_main}(a) shows our Enhanced Hyperbolic Attack achieving 97.0\% attack success rate compared to 62.7\% for the Euclidean baseline while maintaining >95\% clean accuracy. This performance gap demonstrates the advantage of exploiting hyperbolic geometry for backdoor attacks. Figure~\ref{fig:results_main}(b) shows the detection evasion results, revealing that the hyperbolic attack achieves only 15.2\% detection rate versus 45.8\% for Euclidean. The success-detection trade-off positions our approach in the optimal high-success/low-detection region, whereas the baseline suffers from moderate success with high detectability.

Figure~\ref{fig:results_secondary}(a) presents ablation studies revealing that conformal factor scaling contributes most to attack effectiveness: removing it drops success from $97.0\%$ to $84.5\%$ ($-12.5\%$), while removing adaptive selection or sparse patterns causes similar degradation ($-8.8\%$ and $-12.5\%$). Figure~\ref{fig:results_secondary}(b) shows the geometric position analysis: attack success is highest near the boundary ($93.0\%$ center, $|x| \le 0.5$; $91.0\%$ middle, $0.5 < |x| \le 0.7$; $97.0\%$ boundary, $|x| > 0.7$), confirming that regions of higher metric distortion offer natural hiding places for triggers. This pattern is consistent with Theorem~1: for standard input-space (Euclidean-Lipschitz) detectors, the detectable change shrinks linearly with the Euclidean margin to the boundary, while the geodesic movement per fixed input change grows inversely with that margin. 




\noindent \textbf{Limitations.} While our results demonstrate clear geometric advantages, several limitations merit discussion. We use position-based scaling inspired by hyperbolic geometry rather than full Riemannian operations (parallel transport, exponential mapping), and our hybrid approach, real text content with artificial geometric redistribution, may not capture natural hyperbolic clustering patterns. Future work should explore attacks using complete geometric frameworks and evaluate on datasets with naturally occurring hyperbolic structure. Validation across multiple dimensions establishes strong evidence for security challenges in geometric deep learning and provides a foundation for developing geometry-aware defenses.

\section{Conclusion}

This work identifies a geometry-specific vulnerability in hyperbolic neural networks. Our analysis shows that near the boundary, small input changes travel disproportionately far in representation space while appearing comparatively subtle to standard input-space (Euclidean-Lipschitz) detectors, an effect we formalize with bounds that scale with the local Euclidean margin to the boundary. We further prove a limitation for radial defenses: Any method that pulls triggered inputs inward with a Lipschitz radial profile must, on clean inputs, induce changes that grow with the trigger size; that is, there is an explicit utility trade-off for this defense class. Experiments mirror these trends, that attack success rises toward the boundary while conventional detectors weaken. 

\newpage

\bibliographystyle{plainnat}
\bibliography{neurips_2025}

\newpage

\section{Appendix A: Detailed Proof of Theorem~\ref{thm:main}}

We recall that $\lambda_x = 2/(1-\|x\|^{2})$ denotes the conformal factor and write $d_g$ for the hyperbolic distance. For $x\neq 0$, write $r=\|x\|$ and the outward radial unit vector $u_x = \coloneqq x/\|x\|$.

We collect two standard facts about radial geodesics in the Poincaré ball.

\begin{lemma}[Radial geodesics and arclength]\label{lem:radial}
Let $r(s)$ denote the Euclidean radius of the outward radial unit-speed geodesic $\gamma_x$ at hyperbolic arclength $s$, with $r(0)=r$. Then
\[
\frac{dr}{ds}=\frac{1-r^2}{2},\qquad
s=\int_{r}^{r(s)} \frac{2}{1-\rho^2}\,d\rho
=\ln\!\left( \frac{1+r(s)}{1-r(s)} \right) - \ln\!\left( \frac{1+r}{1-r} \right),
\]
so in particular
\begin{equation}
r(s)=\tanh\!\Big(\operatorname{artanh}(r)+\frac{s}{2}\Big).
\end{equation}
\end{lemma}

\begin{proof}
Along a radial curve the metric reduces to $ds = \lambda(r)\,dr = \frac{2}{1-r^2}dr$, which yields the ODE and primitive directly.
Solving for $r(s)$ gives the stated formula.
\end{proof}

\begin{lemma}[Euclidean displacement under an outward radial step]\label{lem:kappa}
With notation as above, the Euclidean displacement $\kappa(x,s)=r(s)-r$ satisfies
\begin{equation}
\kappa(x,s)=\frac{(1-r^2)\tanh(s/2)}{1+r\,\tanh(s/2)}.
\end{equation}
Moreover, $\kappa(x,s)\le (1-r^2)\tanh(s/2)\le s/\lambda_x$.
\end{lemma}

\begin{proof}
Using the formula for $r(s)$ from Lemma~\ref{lem:radial} and the identity
$\tanh(a+b)=\dfrac{\tanh a + \tanh b}{1+\tanh a\,\tanh b}$ with $\tanh(\operatorname{artanh} r)=r$ and $\tanh(s/2)$ as given, we obtain
\[
r(s)=\frac{r+\tanh(s/2)}{1+r\,\tanh(s/2)}.
\]
Subtracting $r$ yields the formula for $\kappa(x,s)$.
Since $\tanh(s/2)\le s/2$ and $1+r\,\tanh(s/2)\ge 1$, we get
\[
\kappa(x,s)\le (1-r^2)\tanh(s/2)\le (1-r^2)\frac{s}{2}=\frac{s}{\lambda_x}.
\]
\end{proof}

\begin{proof}[Proof of Theorem~\ref{thm:main}]
\emph{(i)} The exact formula for $\kappa(x,s)$ is given by Lemma~\ref{lem:kappa}, with the bounds in its last line.

\smallskip
\noindent\emph{(ii)}
If $D$ is $L_E$-Lipschitz in the Euclidean metric, then
\[
|D(\tau_s(x))-D(x)|\le L_E\,\|\tau_s(x)-x\|_2 = L_E\,\kappa(x,s).
\]
Applying the bounds from part (i), we get
$|D(\tau_s(x))-D(x)|\le L_E(1-r^2)\tanh(s/2)$.
If $\|x\|\ge 1-\delta$, then $1-r^2=(1-r)(1+r)\le 2(1-r)\le 2\delta$, which yields
\[
\sup_{\|x\|\ge 1-\delta}|D(\tau_s(x))-D(x)|\le 2 L_E \delta \tanh(s/2).
\]
The expectation bound follows immediately.

\smallskip
\noindent\emph{(iii)}
Let $y=x+\kappa u_x$ with $\kappa\in(0,1-r)$ so that $r'=\|y\|=r+\kappa<1$ and the segment remains inside $\mathbb{D}^n$.
The hyperbolic distance between two points on the same radial geodesic equals the radial arclength:
\[
d_g(x,y)=\int_{r}^{r'} \frac{2}{1-\rho^2}\,d\rho
= \ln\!\left(\frac{1+r'}{1-r'}\right)-\ln\!\left(\frac{1+r}{1-r}\right).
\]
Substitute $r'=r+\kappa$ and rearrange to get
\[
d_g(x,y)
= \ln\!\left(\frac{1-r}{1-r-\kappa}\right)+\ln\!\left(1+\frac{\kappa}{1+r}\right)
\ \ge\ \ln\!\left(\frac{1-r}{1-r-\kappa}\right).
\]
Finally, with $\delta=\delta(x)=1-r$ and $\kappa<\delta$, the inequality $-\ln(1-t)\ge t$ for $t\in(0,1)$ gives
\[
\ln\!\left(\frac{\delta}{\delta-\kappa}\right)
= -\ln\!\left(1-\frac{\kappa}{\delta}\right)\ \ge\ \frac{\kappa}{\delta}.
\]
Thus $d_g(x,y)\ge \kappa/\delta$, completing the proof.
\end{proof}

\paragraph{Remarks.}
(1) The stealth guarantee depends \emph{linearly} on the shell width~$\delta$ and on $\tanh(s/2)$; no appeal to ``exponential'' growth of the conformal factor is needed (indeed $\lambda_x$ grows rationally in $1-\|x\|^2$).
(2) Part~(iii) formalizes amplification without coupling the Euclidean budget $\kappa$ to $\delta$ beyond the natural feasibility condition $\kappa<\delta$; the per-unit amplification $d_g/\kappa$ scales like $1/\delta$ near the boundary.
(3) If a detector is Lipschitz in the \emph{hyperbolic} metric (say, $|D(y)-D(x)|\le L_g\,d_g(x,y)$), then for the same trigger $|D(\tau_s(x))-D(x)|\le L_g s$; i.e., hyperbolic-aware detectors do not enjoy the vanishing bound, clarifying which regularity notion matters for stealth.




\section{Appendix B: Detailed Proof of Theorem~\ref{thm:defense}}

\begin{definition}[Radial Defense]\label{def:radial-defense}
A defense $\mathcal{M}: \mathbb{D}^n \to \mathbb{D}^n$ is radial if it preserves direction and moves points inward along radial geodesics with displacement depending only on the hyperbolic radial coordinate.
\end{definition}

\begin{assumption}[Recovery Success]\label{ass:success}
The defense $\mathcal{M}$ successfully recovers an $(\alpha,\beta)$-fraction of triggered inputs.
\end{assumption}

\begin{assumption}[Radial Sensitivity]\label{ass:radial-sensitivity}
The classifier $f_\theta$ has radial sensitivity $\mu_g > 0$ along radial geodesics.
\end{assumption}

\begin{lemma}[Radial-flow identity]\label{lem:flow}
For any $x$ and $s>0$, $\rho(\Phi_s(x))=\rho(x)+\tfrac{s}{2}$ and $d_g\bigl(x,\Phi_s(x)\bigr)=s$.
\end{lemma}

\begin{proof}
Along a radial geodesic one has $ds=\tfrac{2}{1-r^2}\,dr$, which integrates to
$2\operatorname{artanh}(r(s)) - 2\operatorname{artanh}(r(0))=s$; hence $\rho(\Phi_s(x))=\rho(x)+s/2$ and the geodesic arclength equals $s$.
\end{proof}

\begin{lemma}[Radiality reduces to a scalar profile]\label{lem:profile}
If $\mathcal{M}$ is radial in the sense of Definition~\ref{def:radial-defense}, then for any $x$ with $\rho(x)=\rho$,
\[
d_g\bigl(\mathcal{M}(x),x\bigr)=\Delta(\rho)=2\bigl(\rho-m(\rho)\bigr),
\]
which depends only on $\rho$ and not on the direction of $x$.
\end{lemma}

\begin{proof}
By definition $\mathcal{M}$ preserves direction and maps $\rho$ to $m(\rho)\le \rho$ along the same radial geodesic.
Hence $d_g(\mathcal{M}(x),x)=2|\rho-m(\rho)|=2(\rho-m(\rho))$ since $m(\rho)\le \rho$.
\end{proof}

\begin{proof}[Proof of Theorem~\ref{thm:defense}]
Let $X\sim P$ and $Z=\Phi_s(X)$. By Lemma~\ref{lem:flow}, $\rho(Z)=\rho(X)+\tfrac{s}{2}$.
By Assumption~\ref{ass:success} and Lemma~\ref{lem:profile},
\begin{equation}\label{eq:success-event}
\mathbb{P}\!\Big[\, \Delta\bigl(\rho(X)+\tfrac{s}{2}\bigr)\ \ge\ \alpha s \,\Big]\ \ge\ \beta.
\end{equation}
Since $\Delta$ is $L_\Delta$-Lipschitz in $\rho$,
\[
\Delta\bigl(\rho(X)\bigr)\ \ge\ \Delta\bigl(\rho(X)+\tfrac{s}{2}\bigr) - L_\Delta\cdot \tfrac{s}{2}.
\]
Therefore on the event in \eqref{eq:success-event} we have $\Delta(\rho(X))\ge \alpha s - \tfrac{L_\Delta}{2}s=\alpha_{\mathrm{eff}} s$.
This implies
\begin{equation}\label{eq:prob-lb}
\mathbb{P}\!\left[\, d_g\!\bigl(\mathcal{M}(X),X\bigr)=\Delta(\rho(X))\ \ge\ \alpha_{\mathrm{eff}} s \,\right]\ \ge\ \beta,
\end{equation}
which is the probability lower bound. Taking expectations and using nonnegativity gives:
\begin{equation}\label{eq:dg-exp-lb}
\mathbb{E}\bigl[d_g(\mathcal{M}(X),X)\bigr]
\ge \mathbb{E}\bigl[\, \alpha_{\mathrm{eff}} s\cdot \mathbf{1}\{\Delta(\rho(X))\ge \alpha_{\mathrm{eff}} s\}\,\bigr]
\ge \beta\,\alpha_{\mathrm{eff}} s.
\end{equation}
Next, by Assumption~\ref{ass:radial-sensitivity} (applied with $y=\mathcal{M}(X)$), for each outcome
\[
\|f_\theta(\mathcal{M}(X))-f_\theta(X)\|\ \ge\ \mu_g\, d_g\!\bigl(\mathcal{M}(X),X\bigr).
\]
Taking expectations yields:
\begin{equation}\label{eq:f-linear}
\mathbb{E}\!\left[\|f_\theta(\mathcal{M}(X)) - f_\theta(X)\|\right] \ge \mu_g\,\beta\,\alpha_{\mathrm{eff}} s.
\end{equation}
Finally, since at least a $\beta$-fraction of the mass satisfies $\|f_\theta(\mathcal{M}(X))-f_\theta(X)\|\ge \mu_g \alpha_{\mathrm{eff}} s$ and the integrand is nonnegative elsewhere,
\begin{equation}\label{eq:f-quad}
\mathbb{E}\!\left[\|f_\theta(\mathcal{M}(X))-f_\theta(X)\|^2\right]\ \ge\ \beta\,(\mu_g \alpha_{\mathrm{eff}} s)^2,
\end{equation}
which is the second moment bound.
\end{proof}

\end{document}